\documentclass[12pt]{article}

\usepackage{caption} 
\expandafter\let\csname equation*\endcsname\relax
\expandafter\let\csname endequation*\endcsname\relax 
\usepackage{amsmath,amssymb,amsfonts,siunitx,commath}
\usepackage{amsthm}
\usepackage{tikz,url}
\usepackage{mathtools}
\usepackage{algorithm}
\usepackage{algpseudocode}
\mathtoolsset{showonlyrefs}
\usepackage{enumerate}
\usepackage[utf8]{inputenc} 
\usepackage[T1]{fontenc}    
\usepackage{hyperref}       
\usepackage{url}            
\usepackage{booktabs}       
\usepackage{amsfonts}       
\usepackage{nicefrac}       
\usepackage{microtype}      
\usepackage{subcaption}

\DeclareMathOperator*{\argmax}{arg\,max}
\DeclareMathOperator*{\argmin}{arg\,min}
\newcommand\R{\mathbb{R}}
\newcommand\E{\mathbb{E}}

\newcommand\dx{\mathrm{d}}

\newtheorem{lemma}{Lemma}

\newtheorem{proposition}[lemma]{Proposition}
\newtheorem{remark}[lemma]{Remark}

\usepackage{xr}

\begin{document}
\title{Mixed Noise and Posterior Estimation with Conditional DeepGEM
}
\author{Paul Hagemann\footnotemark[1],\and Johannes Hertrich\footnotemark[2], \and Maren Casfor \footnotemark[3],  \and Sebastian Heidenreich\footnotemark[3], \and Gabriele Steidl\footnotemark[1]}
\date{\today}
\maketitle

\footnotetext[1]{Institute of Mathematics,
TU Berlin,
Straße des 17.~Juni 136, 
 D-10623 Berlin, Germany,
\{hagemann,steidl\}@math.tu-berlin.de.}

\footnotetext[2]{University College London, j.hertrich@ucl.ac.uk}
\footnotetext[3]{Physikalisch Technische Bundesanstalt, maren.casfor@ptb.de, sebastian.heidenreich@ptb.de}
\renewcommand{\thefootnote}{\arabic{footnote}}

\maketitle

\begin{abstract}
We develop an algorithm
for jointly estimating the posterior and the noise parameters in Bayesian inverse problems, which is motivated by indirect measurements and applications 
from nanometrology with a mixed noise model.
We propose to solve the problem by an expectation maximization  algorithm.
Based on the current noise parameters, we learn in the E-step
a conditional normalizing flow that approximates the posterior.
In the M-step, we propose to find the noise parameter updates again by
an expectation maximization algorithm, which has analytical formulas. 
We compare the training of the conditional normalizing flow with the forward and reverse Kullback-Leibler divergence, and show that our model is able to incorporate information from many measurements, unlike previous approaches. 
\end{abstract}

\section{Introduction}

In a variety of healthcare and other contemporary applications, the variables of primary interest are obtained through indirect measurements, such as in the case of Magnetic Resonance Imaging (MRI) and Computed Tomography (CT).  For some of these applications, the reliability of the results is of particular importance. The accuracy and trustworthiness of the outcomes obtained through indirect measurements are significantly influenced by two critical factors: the degree of uncertainty associated with the measuring instrument and the appropriateness of the (forward) model used for the reconstruction of the parameters of interest (measurand). In this paper, we consider Bayesian inversion to obtain the measurand from signals measured by the instrument and a noise model that mimics background noise coming from the instrument and the variation of the measurement, depending  on the forward model.
 Within this framework, we developed an extension of the expectation maximization (EM) algorithm that is able to handle a Bayesian inversion with a measurement noise model.  As a result, we obtain the posterior distribution for the parameters of interest (distribution of the measurand), which is a measure of the reliability of the measurement results. To demonstrate the applicability and effectiveness we apply the algorithm to two real examples in nanometrology, i.e., EUV Scatterometry. 
The key focus of the work is the development of a noise-adapted posterior sampler based on DeepGEM \cite{gao2021deepgem}, which can incorporate information from several measurements simultaneously.

In this context we consider Bayesian inverse problems  
\begin{equation} \label{inverse}
Y_{\theta} = F(X) + \eta_{\theta},
\end{equation}
with a possibly nonlinear forward operator $F\colon \R^d \to \R^n$ and 
a random noise variable $\eta_{\theta}$ which depends on an unknown parameter $\theta$ and on $F(X)$.  Note, $Y$ describes the signals of the instrument whereas $X$ are the parameters of interest.
The posterior (parameter distribution) $P_{X|Y_{\theta}=y}$ for observations $y$ will ultimately depend on these parameters $\theta$, as the likelihood $P_{Y_{\theta}|X=x}$ depends on them. 
Therefore, we aim to estimate the parameter $\theta$ from observations $y_i \in \mathbb{R}^n$, $i=1,...,N$, where $N$ is possibly small.
There exists plenty of literature on estimating the standard deviation $\sigma$ within the Gaussian noise model
$\eta_{\theta} = \eta_{\sigma} \sim \mathcal{N}(0,\sigma^2 I_n)$.
However, motivated by applications in nanometrology \cite{HGB2018}, we are interested in a mixture of additive and multiplicative Gaussian noise of the form 
\begin{equation} \label{noise}
\eta_\theta = \eta_{(a,b)} =
\eta_1 + \eta_2, \quad \eta_1 \sim \mathcal{N}(0, a^2{I_n}),; 
\eta_2 \sim  \mathcal{N}\big(0, b^2 \text{diag} (F(x)^2) \big), 
\end{equation}
where $ F(x)^2 = (F_i(x)^2)_{i=1}^n$ and the identity in $\mathbb{R}^{n\times n}$ is given by $I_n$.
For convenience we assume that the instrument noise and other sources can be described by the simple Ansatz made here. 
In general different noise models may appear in the applications. The noise model Eq.\eqref{noise} was used in several previous studies in optics \cite{HGB2015, HGB2018, herrero2021uncertainties, saadeh2021time} and analyzed in \cite{dunlop2019multiplicative}. A similar noise model appears in analytical chemistry \cite{rocke1995two} and the study of gene expression arrays \cite{rocke2001model}. It belongs to the class of heteroskedastic noise models \cite{foi2009clipped} and an algorithm for parameter estimation in a slightly different problem was proposed in \cite{foulley97}. Learning the noise model without any parametric form was done using NFs in \cite{abdelhamed2019noise}.

The standard approach for parameter estimation is \emph{maximum likelihood estimation}. 
That is, we choose $\theta$ as the minimizer of the negative log likelihood function
$$
\mathcal L(\theta)=\frac{1}{N}\sum_{i=1}^N \log(p_{\theta}(y_i)),
$$
where $p_{\theta}$ is the probability density function of $Y_\theta$.
However, in our case this function involves a high-dimensional integral of the form $\int_x p_{Y_{\theta}|X=x}(y) dP_X(x)$ which is intractable to compute. 
As a remedy, we exploit EM algorithms which were introduced in \cite{DLR1977}, see also \cite{BN2006} for an overview and \cite{emforhyper,Nan_2020_CVPR} for applications for parameter estimation in inverse problems.
The basic idea is to iteratively compute the posterior distribution $P_{X|Y_{\theta^{(r)}}=y_i}$ for a current estimate $\theta^{(r)}$ of $\theta$ and then updating this estimate to $\theta^{(r+1)}$ based on this posterior distribution.
Here, the computation of $P_{X|Y_{\theta^{(r)}}=y_i}$ is called E-step, while the update of $\theta$ is called M-step.
Intuitively, this corresponds to the idea that the distribution of $F(X)$ is approximately concentrated on a lower-dimensional manifold and consequently the distance of the $y_i$ to this manifold contains the information of the noise parameters $\theta$.

Recently, Gao et al.~\cite{gao2021deepgem} proposed to solve the E-step by a normalizing flow (NF) \cite{dinh2017density} using the reverse/backward Kullback-Leibler (KL) divergence as loss function. 
For the M-step, they apply a stochastic gradient ascent algorithm.
Note that in general the same procedure can be applied to estimate parameters of the forward operator instead of the noise level, see \cite{Laroche_2024_WACV}.
This approach has several drawbacks. The model needs to be retrained for every new observations $y$ and cannot profit from many observations that follow the same error parameters. 
Furthermore, the reverse KL is known to be \emph{mode-seeking}. That is, it tends to recover only one mode of multimodal distributions, which incorporates a significant approximation error, see the discussions in \cite{HHS2021, minka2005, WKN2020} for more details.  Very related is also the JANA framework \cite{pmlr-v216-radev23a}, where the authors propose to learn forward (likelihood), posterior and summary network together in an amortized (i.e. conditional manner). However, they optimize it not iteratively, since they do not treat it as an EM framework and do not discuss noise modelling. The main idea of this paper was presented by some of the authors in a one-page extended abstract in \cite{ramos_2023_8393048}.

\paragraph{Contributions}
First, we propose to use the conditional normalizing flows \cite{ALKRK2019,winkler2023learning} in the E-step.
This allows the incorporation of \emph{several} measurements from the same error model and to solve the inverse problem for all measurements \emph{simultaneously}. 
Fortunately, the forward KL \cite{ALKRK2019} can be used as loss function for training the conditional NFs which makes the method mode covering.
Second, we propose an inner EM algorithm for solving the M-step more efficiently.
For our special noise model \eqref{noise}, we deduce analytic expressions for 
E- and M-steps of this inner algorithm. The performance of our approach will be demonstrated on two applications from nano-optics. In particular, we propose a conditional version of DeepGEM and benchmark it against forward conditional DeepGEM, where the reverse KL is replaced by a forward KL.

\paragraph{Organization}
We start in Section \ref{sec:em} by recalling the general EM algorithm.
Then, in Section \ref{sec:parameter_est}, we construct the E-step and M-step
for our application.
That is, we show how conditional normalizing flows can be incorporated into the E-step and describe how the M-step can be solved for our noise model
with an ``inner'' EM algorithm which steps can be given in a closed analytical form.
Some of the technical computations are postponed to \ref{m_step_deriv}.
We test our algorithms on two nano-optics problems, which is done in Section \ref{sec:exp}.
Finally, conclusions are drawn in Section \ref{sec:conc}.

\section{EM Algorithm} \label{sec:em}
In this section, we introduce the EM algorithm as a 
maximization-maximization algorithm of an evidence lower bound.
A general introduction into the EM algorithm can be found, e.g.,~in \cite{BN2006}.

Let $\{Y_\theta:\theta\in\Theta\}$ be a family of $n$-dimensional random variables 
having probability density functions $p_\theta$, $\theta\in\Theta$.
Given i.i.d.~samples $y_1,...,y_N\in\R^n$ from $Y_{\theta^*}$  for some unknown $\theta^*$, which we want to approximate by computing the maximum log-likelihood
estimator 
$$
\hat \theta=\argmax_{\theta\in\Theta} \mathcal L(\theta),\quad \mathcal L(\theta)\coloneqq \frac1N\sum_{i=1}^N \log(p_{\theta}(y_i)).
$$
In the literature, the term $\log(p_\theta(y))$ is also called \emph{evidence} of $y$ under $\theta$.
As in many applications it is hard to maximize $\mathcal L$, we introduce an absolute continuous
$d$-dimensional auxiliary random variable $X$ such that the joint density $p_{X,Y_\theta}$ exists and is easy to evaluate. Then, it holds by the law of total probability and Jensen's inequality, for any probability density function $q$ on $\R^d$, that
\begin{align}
\log(p_\theta(y))&=\log\Big(\int_{\R^d}p_{X,Y_\theta}(x,y)\dx x\Big)
=\log\Big(\int_{\R^d}\frac{p_{X,Y_\theta}(x,y)}{q(x)} q(x)\dx x\Big)\\
&\geq\int_{\R^d}\log\Big(\frac{p_{X,Y_\theta}(x,y)}{q(x)}\Big) q(x)\dx x \eqqcolon \mathcal F(q,\theta|y).
\end{align}
We call the random variable $X$ the hidden variable  and the expression $\mathcal F(q,\theta|y)$ the \emph{evidence lower bound} (ELBO).
Now, instead of maximizing the log-likelihood function directly, the EM algorithm is a maximization-maximization algorithm for the ELBO, i.e.,~starting
with an initial estimate $\theta^{(0)}$, it
consists of the following two steps:
\begin{align}
\text{E-step:} &\qquad q_i^{(r+1)}=\argmax_{q\in \mathrm{pdf}(d)} \mathcal F(q,\theta^{(r)}|y_i),\quad i=1,...,N\label{eq_estep},\\
\text{M-step:} &\qquad
\theta^{(r+1)} =\argmax_{\theta\in\Theta} \sum_{i=1}^N\mathcal F(q_i^{(r+1)},\theta|y_i)\label{eq_mstep},
\end{align}
where $\mathrm{pdf}$ is the space of $d$-dimensional probability density functions. 

The E-step \eqref{eq_estep} can be solved based on the following standard lemma which can be found, e.g., in \cite[Section 9.4]{bishop}. For convenience, we provide the simple proof. 
Recall that the \emph{Kullback-Leibler} (KL) \emph{divergence} of two probability measures $P,Q$ with densities $p,q$ is defined by
$$
\text{KL}(P,Q) \coloneqq  \int_{\R^d} p(x) \log \frac{p(x)}{q(x)} \, \dx x ,
$$
if $P$ is absolutely continuous with respect to $Q$, and 
$\text{KL}(P,Q) = + \infty$ otherwise. Further, we use the convention $0 \log 0 = 0$.

\begin{lemma}\label{lem_em_kl}
Let $X \in \R^d$ be a absolute continuous  random variable 
and let $Q$ be an absolutely continuous measure on $\R^d$ with probability density function $q$.
Then it holds, for any $y\in\R^n$, that
$$
\log(p_\theta(y)) - \mathcal F(q,\theta|y)
=
\mathrm{KL}(Q,P_{X|Y_{\theta}=y}).
$$
\end{lemma}

\begin{proof}
By definition of the conditional distribution, we have
\begin{align*}
&\quad\log(p_\theta(y)) - \mathcal F(q,\theta|y)\\
&=\log(p_\theta(y)) - \int_{\R^d}q(x)\log\Big(\frac{p_{X|Y_\theta=y}(x)p_{\theta}(y)}{q(x)}\Big)\dx x\\
&=\log(p_\theta(y))-\int_{\R^d}q(x)\log(p_{\theta}(y))\dx x+\int_{\R^d}q(x)\log\Big(\frac{q(x)}{p_{X|Y_\theta=y}(x)}\Big)\dx x\\
 &=\log(p_\theta(y))-\log(p_{\theta}(y))+\mathrm{KL}(Q,P_{X|Y_{\theta}=y})=\mathrm{KL}(Q,P_{X|Y_{\theta}=y}).
\end{align*}
\end{proof}

As the KL divergence $\mathrm{KL}(Q,P)$ is minimal if and only if $Q=P$,
the lemma implies that the solutions $q_i$ of the E-step \eqref{eq_estep} are given explicitly by 
\begin{equation}\label{eq_estep_sol}
\text{E-step:} \qquad q_i^{(r+1)}=
p_{
X | Y_{\theta^{(r)}}=y_i},\quad i=1,...,N.
\end{equation}

For solving the M-step \eqref{eq_mstep}, we decompose the ELBO as
\begin{equation}\label{eq_elbo_decomp}
\mathcal F(q,\theta|y)=\E_{x\sim q}[\log(p_{X,Y_\theta}(x,y))]-\E_{x\sim q}[\log(q(x))].
\end{equation}
Note, that only the first summand depends on $\theta$.
Using this decomposition and the explicit form \eqref{eq_estep_sol} of the $q_i$ in the M-step,
we obtain the classical EM algorithm as proposed in \cite{DLR1977}:
\begin{equation} \label{em}
\text{M-step:} \qquad \theta^{(r+1)}=\argmax_{\theta\in\Theta} \mathcal Q(\theta,\theta^{(r)}),
\end{equation}
\begin{align}
\text{E-step:} \qquad \mathcal Q(\theta,\theta^{(r)})&=\sum_{i=1}^N \E_{x\sim P_{X|Y_{\theta^{(r)}}=y_i}}[\log(p_{X,Y_\theta}(x,y_i))].
\end{align}
A convergence analysis of the EM algorithm based on KL proximal point algorithms was done
in \cite{CH2000,CH2008}.
In particular, we obtain the following convergence properties.

\begin{proposition}\label{prop_conv_EM}
Let the sequence $(q^{(r)},\theta^{(r)})$ be generated by the EM algorithm. Then, the following holds true.
\begin{enumerate}[(i)]
    \item The sequence of ELBO-values $(\mathcal F(q^{(r)},\theta^{(r)}))_r$ is monotone increasing.
    \item The sequence of likelihood values $\mathcal L(\theta^{(r)})$ is monotone increasing.
\end{enumerate}
\end{proposition}

\begin{remark}[Generalized EM algorithms]\label{rem_gen_EM}
Several papers propose so-called generalized EM algorithms \cite{HFL2021,LR1994,MR1993,NH1998}. The key idea of these generalizations
is to replace the maximization steps \eqref{eq_estep} and \eqref{eq_mstep} by increase steps.
More precisely, in each iteration the values $q_i^{(r+1)}$ and $\theta^{(r+1)}$ are chosen such that
$$
\mathcal F(q_i^{(r+1)},\theta^{(r)}|y)\geq \mathcal F(q_i^{(r)},\theta^{(r)}|y), \quad \sum_{i=1}^N \mathcal F(q_i^{(r+1)},\theta^{(r+1)}|y)\geq\sum_{i=1}^N \mathcal F(q_i^{(r+1)},\theta^{(r)}|y).
$$
Using such increase steps, generalized EM algorithms often achieve simpler and faster steps, 
than the original EM algorithm even though they might require more steps until convergence.
By construction, part (i) of Proposition~\ref{prop_conv_EM} remains for generalized EM algorithms,
while part (ii) is not longer proved for certain of these algorithms.
\end{remark}

\section{Parameter Estimation in Bayesian Inverse Problems} \label{sec:parameter_est}
Now we consider the inverse problem \eqref{inverse}, where
 we assume that $X$ has density $p_X$.
Given $N$ observations $y_1,...,y_N$ of $Y_\theta$, 
we aim to determine the parameter $\theta$.
We will derive an EM algorithm for this problem, where 
the hidden variable is given by the ground truth random variable $X$.
In particular, we will deal with the noise model \eqref{noise}. Here the parameter 
$\theta = (a,b)$ can be updated in the M-step analytically.

\subsection{E-Step: Conditional NFs}\label{sec:e_cnf}
As we have seen in \eqref{eq_estep_sol}, 
the E-step corresponds 
to finding the posterior densities 
$$
p_{X|Y_{\theta^{(r)}}=y_i}=\argmin_q\mathcal F(q,\theta^{(r)}|y_i), \quad i=1,\ldots,N,
$$
for given $\theta^{(r)}$.
We propose to approximate these posteriors by conditional NFs.
This extends the so-called DeepGEM from \cite{gao2021deepgem} to the conditional case.
We will see that our approach brings the forward KL instead of the reverse KL 
into the play which has several advantages, see Remark \ref{rem:kl}.

A conditional NF is a mapping $\mathcal T_\phi\colon\R^n\times\R^d\to\R^d$ depending on some parameters $\phi$ such that $\mathcal T_\phi(y,\cdot)$ is invertible for any $y\in\R^n$.
In this paper, $\mathcal T_\phi$ is a neural network. 
There were several architectures for NFs proposed in the literature. They include GLOW \cite{KD2018}, real NVP \cite{dinh2017density}, invertible ResNets \cite{BGCDJ2019,CBDJ2019,H2023} and autoregressive Flows \cite{CTA2019,DBMP2019,HKLC2018,PPM2017}.
They were extended to the conditional setting in \cite{ALKRK2019,DSLM2021,HHS2021generalized_nf,KDSK2020, winkler2023learning}.
The parameters $\phi$ are learned such that 
$$
P_Z \circ T_\phi(y_i,\cdot)^{-1} =  {\mathcal T_\phi(y_i,\cdot)}_\#P_Z\approx P_{X|Y_{\theta^{(r)}}=y_i},
$$
for all $i=1,...,N$, where $P_Z$ is some latent distribution, usually a standard Gaussian one.
Once we have learned the conditional NF $\mathcal T_\phi$ for an appropriate parameter $\theta$
this provides us with a desired approximation of the posterior.
In other words, given $y \in \R^n$, we can sample $z$ from $Z$ and produce a sample
from $P_{X|Y=y}$ by $\mathcal T_\phi(y,z)$.

Now we could learn the conditional NF $\mathcal T_\phi$ by minimizing the loss function
$$
\mathcal J_\text{reverse}(\phi)=\sum_{i=1}^N\mathcal F(p_{\mathcal T_\phi(y_i,\cdot)_\#P_Z},\theta^{(r)}|y_i)\propto\sum_{i=1}^N\mathrm{KL}(\mathcal T_\phi(y_i,\cdot)_\#P_Z,P_{X|Y_\theta=y_i}),
$$
where the last relation follows from Lemma~\ref{lem_em_kl} and ``$\propto$'' indicates equality up to a constant.
In literature, this loss function is known as \emph{reverse} or \emph{backward} KL loss function. Applying the change of variable formula for push-forward measures and Bayes's formula 
this can be rewritten as
\begin{align*}
\mathcal J_\text{reverse}(\phi)
&\propto\sum_{i=1}^N\mathrm{KL}(P_{Z},\mathcal T_\phi(y_i,\cdot)^{-1}_\#P_{X|Y_{\theta^{(r)}}=y_i})\\
&\propto - \sum_{i=1}^N \E_{z\sim P_Z}
\big[\log(p_{Y_{\theta^{(r)}}|X=\mathcal T_\phi(y_i,z)}(y_i))+\log(p_X(\mathcal T_\phi(y_i,z)))\\
&\qquad+\log(|\mathrm{det}(\nabla_z\mathcal T_\phi(y_i,z))|) \big],
\end{align*}
see \cite{AH2023,AFHHSS2021,KDSK2020} for a detailed explanation and applications.
In order to evaluate these terms we have to be able to evaluate the prior density
$p_X$ as well as the conditional densities $p_{Y_{\theta^{(r)}}|X=x}(y)$, which contains the forward operator
and the noise model for given parameters $\theta^{(r)}$.
Unfortunately, it is known from the literature that the reverse KL is prone to mode collapse, see \cite{minka2005}.
That is, in the case that $P_{X|Y_\theta=y_i}$ is multimodal, it tends to generate only samples from one
of the modes.

As a remedy, we interchange the arguments in the KL divergence in $\mathcal J_\text{reverse}$ and replace the sum over the $y_i$ by the expectation over $P_{Y_{\theta^{(r)}}}$. Then, we arrive at the so-called \emph{forward} KL loss function
\begin{align}
\mathcal J_\text{forward}(\phi)
&=
\E_{y\sim P_{Y_{\theta^{(r)}}}}
\big[\mathrm{KL}(P_{X|Y_{\theta^{(r)}}=y},\mathcal T_\phi(y,\cdot)_\#P_Z) \big]
\\
&=
\E_{y\sim P_{Y_{\theta^{(r)}}}}
\Big[
\E_{x \sim P_{ X|Y_{\theta^{(r)}} =  y}}
\big[
\log \Big(\frac{ p_{X|Y_{\theta^{(r)}} = y}(x) }{ p_{\mathcal T_\phi(y,\cdot)_\# P_Z }(x)}\Big)
\big]
\Big]
\\
&
\propto-\E_{(x,y)\sim P_{X,Y_{\theta^{(r)}}}}\big[\log(p_Z(\mathcal T_\phi(y,\cdot)^{-1}(x)))+\log|\mathrm{det}(\nabla\mathcal T_\phi(y,\cdot)^{-1}(x))|\big].
\end{align}
To compute these terms we need samples 
$(\tilde x_j,\tilde y_j)$, $j=1,...,\tilde N$ 
from the joint distribution 
$P_{X,Y_{\theta^{(r)}}}$. 
Note that such samples can be generated from just knowing the 
$\tilde x_j$ by evaluating the forward operator and the noise model. In this setting, we do not need access to the prior density $p_X$ or the conditional densities 
$p_{Y_{ \theta^{(r)}}|X=x }$. The forward KL is more standard in (conditional) generative modelling \cite{ALKRK2019, dinh2017density, winkler2023learning} due to these properties and is also known as maximum likelihood training \cite{WKN2020}.

\begin{remark}[Forward versus Reverse KL]\label{rem:kl}
Note that in the case that $\mathcal T_\phi$ is an universal approximator, we have for both loss functions that the optimal parameters $\hat \phi$ fulfills $\mathcal T_{\hat \phi}(y_i,\cdot)_\#P_Z=P_{X|Y_{\theta^{(r)}}=y_i}$. This is important, as we propose to replace the reverse KL in the E-step by the forward KL. 
Moreover, the assumptions for training and the approximation properties differ.
For the reverse KL, we have to be able to evaluate the density $p_X$ of the prior distribution, while the forward KL needs samples from $P_X$. In practice it depends on the problem which assumption is more realistic.
On the other hand, the forward KL loss function is not that prone to mode collapse. The universality of conditional normalizing flows has been discussed in \cite{lyu2022paracflows}.
\end{remark}

\subsection{M-Step: Inner EM for Mixed Noise Model}\label{sec:m_em}
As described in \eqref{em}, the M-step is given by
\begin{align}
\theta^{(r+1)}
&=
\argmax_{\theta}\sum_{i=1}^N\E_{x\sim P_{X|Y_{\theta^{(r)}}=y_i}}[\log(p_{X,Y_{\theta}}(x, y_i))].\label{eq:m_opt-prob}
\end{align}
Unfortunately, to the best of our knowledge, an analytic solution of \eqref{eq:m_opt-prob} is not available.
Therefore, we discretize the expectation in \eqref{eq:m_opt-prob} by
\begin{align}
\theta^{(r+1)}=\argmax_{\theta}\sum_{i=1}^N\sum_{k=1}^M\log(p_{X,Y_{\theta}}(x_{i}^k,y_i)),\label{eq:discrete_m_step}
\end{align}
where the $x_{i}^k$, $k=1,...,M$ are sampled from $P_{X|Y_{\theta^{(r)}}=y_i}$.
This can be solved by various iterative methods, e.g., by a stochastic gradient algorithm \cite{kingma2017adam} as done in \cite{gao2021deepgem}.

For our special noise model \eqref{noise} with $\theta = (a,b)$,
we propose to use again an  EM algorithm, 
since both the E- and M-step of the ``inner'' EM can be computed analytically, which will be shown in the following paragraphs. 
We use that for this noise model we have
$$
p_{Y_{\theta}|X=x}(y)=\mathcal N(y|F(x),a^2 I_n+b^2\mathrm{diag}(F(x))).
$$
For simplicity, we assume that we have only $M=1$ samples. 
The case $M>1$ can be reduced to this case by considering $M$ copies of $y_i$.
In our EM algorithm for \eqref{eq:discrete_m_step} we use $V_\theta \sim  a \mathcal N(0,1)$
as hidden variable, which corresponds to the ``additive part'' of the noise.

\paragraph{Inner E-step}
We have to compute the conditional distribution $P_{V_\theta|(X,Y_\theta)=(x,y)}$. 
Using Bayes' formula, we obtain
\begin{align}
\log(p_{V_\theta|(X,Y_\theta)=(x,y)}(v))\propto& \log(p_{Y_\theta|V_\theta=v,X=x}(y))+\log(p_{V_\theta|X=x}(v))\\
\propto& \sum_{j=1}^n-\frac{(y_j-F_j(x)-v_j)^2}{2 b^2F_j(x)^2}-\frac{v_j^2}{2a^2}\\
\propto& \log\Big(\mathcal N\Big(v\Big|\tfrac{a^2(y-F(x))}{a^2+b^2F(x)^2},\mathrm{diag}\big(\tfrac{a^2b^2 F(x)^2}{a^2+b^2 F(x)^2}\big)\Big)\Big),
\end{align}
where the quotients in the last line are understood componentwise and "$\propto$" indicates, that we have equality up to an additive constant independent of $v$.
Consequently, the conditional distribution $P_{V_\theta|(X,Y_\theta)=(x,y)}$ is given by
$$
P_{V_\theta|(X,Y_\theta)=(x,y)} = \mathcal N\Big(\tfrac{a^2(y-F(x))}{a^2+b^2F(x)^2},\mathrm{diag}\big(\tfrac{a^2b^2 F(x)^2}{a^2+b^2 F(x)^2}\big)\Big).
$$

\paragraph{Inner M-step:}
We will just outline the final result. 
The quite technical proof is deferred to appendix \ref{m_step_deriv}, where 
the M-step can be rewritten as
$$
\theta^{(r+1)}=\argmax_{\theta = (a,b) \in \R^2_{\ge 0}}A_1(b)+A_2(a),
$$
where
$$
A_1(b) =\frac{1}{2b^2}c_1^{(r)}-n+\log(b)+\mathrm{const},\quad A_2(a)=\frac{1}{2a^2}c_2^{(r)}-n\log(a)+\mathrm{const},
$$
with
\begin{align}
c_1^{(r)}&=-\frac{1}{N}\sum_{i=1}^N\sum_{j=1}^n\frac{(y_{ij}-F_j(x_i))^2(b^{(r)})^4F_j(x_i)^2}{((a^{(r)})^2+(b^{(r)}F_j(x_i))^2)^2}+\frac{(a^{(r)}b^{(r)})^2}{(a^{(r)})^2+(b^{(r)}F_j(x_i))^2}\\
c_2^{(r)}&=-\frac{1}{N}\sum_{i=1}^N \bigg(\sum_{j=1}^n \Big(\frac{(a^{(r)})^2(y_{ij}-F_j(x_i))}{(a^{(r)})^2+(b^{(r)})^2F_j(x_i)^2}\Big)^2+\frac{(a^{(r)}b^{(r)}F_j(x_i))^2}{(a^{(r)})^2+(b^{(r)}F_j(x_i))^2}\bigg),
\end{align}
and $y_i=(y_{i1},...,y_{in})$ and $F = (F_1,...,F_n)\colon\R^d\to\R^n$.
By setting the derivatives of $A_1$ and $A_2$ to zero, this is equivalent to
$$
a^{(r+1)}=-\frac{c_2^{(r)}}{n},\qquad b^{(r+1)}=-\frac{c_1^{(r)}}{n}, 
$$
which are the update rules we will use. 

\subsection{Resulting Algorithm}

The summary of the two nested algorithms can be seen in Algorithm~\ref{alg_cnf_mixed}. Here, both the E-steps and the M-steps are not run for 1 iteration, but several. In particular the analytical M-step is cheap, and therefore it is intuitive to make use of this. For the E-step we take usually 10 steps to perform posterior updates. The initialization of $a,b$ are done in such a way that we approximate the posterior distribution ``from above''. This is important so that the observed measurements are included in the distribution $P_{Y_{\theta}}$, which is similar to the logdet schedule proposed in \cite{sun2021deep}. It indeed can be shown that making the logdet term larger corresponds to scaling the noise higher for additive Gaussian noise, which makes the estimated distributions broader and therefore prevents mode missing, although we still found that this does not solve the problem completely.

\begin{algorithm}[!ht]
\caption{EM Algorithm Mixed Noise estimation via CNFs}\label{alg_cnf_mixed}
\begin{algorithmic}
\State Input:  $y_1,...,y_N\in\R^{n}$, conditional normalizing flow $\mathcal{T}$ and
initial estimate $a^{(0)},b^{(0)}$, number of $x$ in total $K = 2000 \geq N$.
\For {$r=0,1,..., R$}
\State \textbf{E-Step:} \For {$p=0,1,..., P$} update
 $\mathcal{T}$  according to 
 $$
 \argmin_{\phi}\E_{y\sim P_{Y_\theta}}[\mathrm{KL}(p_{X|Y_\theta=y},p_{\mathcal T_\phi(y,\cdot)_\#P_Z})]\qquad\text{(reverse KL)}
 $$
 \State or 
 $$
 \argmin_{\phi}\E_{y\sim P_{Y_\theta}}[\mathrm{KL}(p_{\mathcal T_\phi(y,\cdot)_\#P_Z},p_{X|Y_\theta=y})]\qquad\text{(forward KL)}
 $$ 
 \State via a gradient descent method where $\theta = (a^r,b^r).$ \\
 \EndFor
\\
\State  Repeat $\tilde{y} = (y_1,..,y_1,...,y_N,...y_N) \in \mathbb{R}^{K\ n}$.
\State  Sample $x_i \sim \mathcal{T}(\tilde{y}_i,z_i)$ for standard normal $z_i$. 

\State \textbf{M-Step:} 
 \For {$l=0,1,..,L$}
\begin{align}
c_1^{(r,l)}&=-\frac{1}{K}\sum_{i=1}^K\sum_{j=1}^n\frac{(\tilde{y}_{ij}-F_j(x_i))^2(b^{(r)})^4F_j(x_i)^2}{((a^{(r)})^2+(b^{(r)}F_j(x_i))^2)^2}+\frac{(a^{(r)}b^{(r)})^2}{(a^{(r)})^2+(b^{(r)}F_j(x_i))^2}\\
c_2^{(r,l)}&=-\frac{1}{K}\sum_{i=1}^K \bigg(\sum_{j=1}^n \Big(\frac{(a^{(r)})^2(\tilde{y}_{ij}-F_j(x_i))}{(a^{(r)})^2+(b^{(r)})^2F_j(x_i)^2}\Big)^2+\frac{(a^{(r)}b^{(r)}F_j(x_i))^2}{(a^{(r)})^2+(b^{(r)}F_j(x_i))^2}\bigg).
\end{align}
\State Update
$
(a^{(r,l)})^2=-\frac{c_2^{(r,l)}}{n},\qquad (b^{(r,l)})^2=-\frac{c_1^{(r,l)}}{n}
$
\EndFor
\State Set $(a^{(r+1)})^2=(a^{(r,L)})^2,\qquad (b^{(r+1)})^2=(b^{(r,L)})^2 $
\EndFor
\end{algorithmic}
\end{algorithm}

\section{Experiments}\label{sec:exp}
We will benchmark our algorithm on two problems from nano-optics, the first one being low-dimensional and the second one harder and more recent. The first was introduced in \cite{HGB2015,HGB2018} and the second one is part of a current research project. The goal of this is to learn both a reasonable posterior reconstruction as well as the error parameters $a,b$ jointly. To showcase the advantages of making the models conditional we also vary the number of measurements and hope that more measurements lead to better reconstructions. 

Generally, we use the PyTorch framework \cite{pytorchpaper} and use FrEia package for implementation of the conditional normalizing flows \cite{freia}. The code is available on GitHub~\footnote{\url{https://github.com/PaulLyonel/ConditionalDeepGEM}}. 
We train our models using the Adam optimizer \cite{kingma2017adam} and fix some hyperparameter choices across the experiments. In particular, we only use a learning rate of 1e-3, $P = 10$, set $K= 2000$, $R = 5000$ in \ref{alg_cnf_mixed} and $L = 20$. The choice of $K$ is in particular constant no matter how many measurements $N$ are used. This allows us to compare whether the information of many measurements is beneficial for the estimation of $a$ and $b$. However, these hyperparameters were not optimized in a grid search and therefore it is likely that one can improve the performance. We generate synthetic measurements via surrogate forward operators with known noise levels $a_{true}, b_{true}$, similar as in \cite[Section 3]{AFHHSS2021}.  In both these experiments, we take these surrogate neural networks as forward operators. The extension to real world measurements and the relation to the true PDE inverse problem is left for future work. This allows us, given some noise parameters, to sample $(x,y)$ data on the fly.

Then we apply our proposed algorithms to learn a and b as well as the posterior reconstructions. Then we are able to compare the models and error parameters on two metrics. The metrics and models evaluated are summarized below. 

\paragraph{Models evaluated}
We will evaluate two EM-based models, one is the conditional version of the DeepGEM method \cite{gao2021deepgem} which we combined with our M-step. Note that amounts to using the reverse KL divergence in algorithm \ref{alg_cnf_mixed}. However we propose to use the forward KL divergence which we call conditional forward DeepGEM. To see the general comparison of EM algorithms, we also implement a grid search over $(a,b)$ and save the ``best'' model of this. The grid search is still possible since we are searching over a two dimensional space, but becomes quickly infeasible for higher dimensional noise models. We will call this grid conditional NF and also evaluate its forward and reverse KL version. 

\paragraph{Metrics} We are going to benchmark the models using the following two metrics. 
\begin{itemize}
\item \emph{Distance to true a and b}: We will consider synthetic data, where the $a,b$ the observations were generated with, are known. This metric is given by $$D((a,b), (a_{true}, b_{true})) = \frac{\vert a - a_{true} \vert}{a_{true}} + \frac{\vert b - b_{true} \vert}{b_{true}}.$$ However, this is a terrible metric as there can be other combinations of a and b which explain the observations equally well. However, we hope that for sufficient observations we will converge to the true a and b. 
\item \emph{ELBO}: From lemma \ref{lem_em_kl} we see that maximizing the ELBO $\mathcal{F}$ leads to minimizing the KL distance to the true posteriors as well as maximizing the likelihoods of the observations under the estimated error parameters $a,b$. This is a good proxy, as both the likelihood of the observations as well as the KL distance to the posteriors are intractable in high dimensions. 
\end{itemize}
By the above discussion, we also obtain a suitable \emph{model selection} criterion. We train all the models for the same amount of steps, but we validate it after every EM-step according to the ELBO for the measurements. We then load the best model for every run and evaluate our metrics. 

\subsection*{Scatterometry}
For chip manufacturing the control of nanopattern in the lithography process is essential and non-destructive measurement methods with high throughput are desirable. In addition to standard scanning electron microscopy (low throughput, destructive) scatterometry is gaining importance.  
Scatterometry is a non-destructive optical measurement technique for assessing lithography's periodic nanostructures' critical dimensions (CDs) \cite{huang2004spectroscopic}. In this measuring method, nanostructured periodic surfaces are illuminated with light and refraction patterns are detected. From these patterns geometry parameters are reconstructed by solving an inverse problem. According to Eq.~\eqref{inverse} observations are given by the refraction patterns, the forward operator is determined by time-harmonic Maxwell’s equations and the noise is given by the instrument as well as the model error. 

In the following, we consider two examples to demonstrate the performance of the developed algorithm for applications in nanometrology of chip production. The first example considers a typical photomask for extreme ultra violet light (EUV) and the second a line grating.

\subsubsection*{Photo mask}
The EUV-photomask considered here consists of periodic absorber lines, capping layers and a multilayer stack functioning as a mirror for 13.4 nm wavelength waves (EUV range). Key geometry parameters include the line width, height and the angle of the sidewall (3 parameters). The refraction patterns comprise 23 intensities (maxima of the refractive orders) and the measurement/model noise is assumed to be distributed according to our mixed noise model.

The problem has $x$-dimension $3$ and $y$-dimension $23$ and therefore is well-suited for first experiments. Furthermore, by \cite{HGB2018} it is known that the posterior is indeed multimodal. The prior is chosen uniformly in $[-1,1]$ and its density is approximated like in \cite{HHS2021} for the reverse KL.
For the example we train both the conditional DeepGEM as well as the conditional forward DeepGEM using data from the finite element method (FEM) based forward model \cite{HGB2018}, which is approximated by a surrogate neural network. The forward DeepGEM is a bit quicker to train. The true a and b used to generate simulated signals of the instrument  were set to $0.005$ and $0.1$ respectively. We benchmark now the four methods, the conditional DeepGem with forward and reverse KL as well as the grid conditional normalizing flows (gridCNF) with forward and reverse KL. For the grids we chose an equispaced grid with 8 points for a and for b. For a this ranged from 0.001 to 0.03 and for b from 0.01 to 0.2. Concerning training time, the forward and reverse conditional DeepGEM were similar with 9 minutes per run. The grid versions took approximately 13 minutes to train, where we took 1200 optimizer steps per grid point. Generally, we can see that the grid methods get outperformed by our EM versions, although they take a longer time. 
From Table \ref{old_ab}  and Fig.\ref{fig:err} we can see that the forward KL and the reverse KL have both similar performance in terms of distance to the true a and b, where the forward KL seems to have a slight edge in the case of many measurements. However, in terms of ELBO, we observe in  Table  \ref{old_elbo} that the forward KL performs favorably.
This is somewhat remarkable, as the reverse KL is the ELBO objective when ignoring the parameters independent of the flow.  Considering posterior measures obtained form simulated measurements 
we realize that the reverse KL 
does not exactly reproduce the modes in some of the examples, see Fig. \ref{fig_old} whereas the 
forward KL performs quite well. The inability
of the reverse KL to detect the correct modes of the posterior can indeed explain the better performance of the forward conditional DeepGEM. Both algorithms are improving with more measurements, see Fig. \ref{fig:forward_error} and \ref{fig:rev_error}.

\begin{table}[]
\centering
\begin{tabular}{ |c||c|c|c|c| } 
 \hline
 number measurements &1 &2 &4 &8 \\ \hline
 forward condDeepGEM & \textbf{0.50} & 0.37  & 0.33 &  \textbf{0.20}\\ 
 reverse condDeepGEM & 0.60 & 0.36  & \textbf{0.31} & 0.22 \\
 forward gridCNF & 0.65 & 0.41 & 0.44 & 0.32 \\
 reverse gridCNF & 0.53 & \textbf{0.34}  & \textbf{0.31} & 0.42 \\
 \hline
\end{tabular}
 \caption{Distance of estimated $a$ and $b$ to the true ones over 10 runs.}
 \label{old_ab}
\end{table}
\begin{table}
\centering
\begin{tabular}{ |c||c|c|c|c| } 
 \hline
 number measurements &1 &2 &4 &8 \\ \hline
 forward condDeepGEM & \textbf{81.12} & \textbf{80.55}  & \textbf{79.36} &  \textbf{79.20}\\ 
 reverse condDeepGEM & 80.92 & 79.80  & 78.62 & 78.59 \\
 forward gridCNF & 77.54 & 76.89 & 76.33 & 76.61 \\
 reverse gridCNF & 78.99 & 77.55  & 76.32 & 76.50 \\
 \hline
\end{tabular}
 \caption{ELBO of the algorithms over 10 runs. Calculated for the measurements based on 2000 samples. }
\label{old_elbo}
\end{table}

\begin{figure}
\begin{subfigure}[t]{.49\textwidth}
\includegraphics[width=\linewidth]{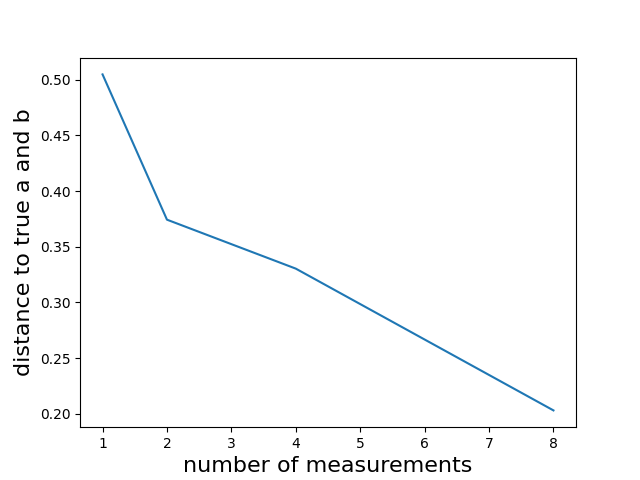}
        \caption{Forward KL conditional DeepGEM.}
        \label{fig:forward_error}
        \end{subfigure}
\begin{subfigure}[t]{.49\textwidth}
\includegraphics[width=\linewidth]{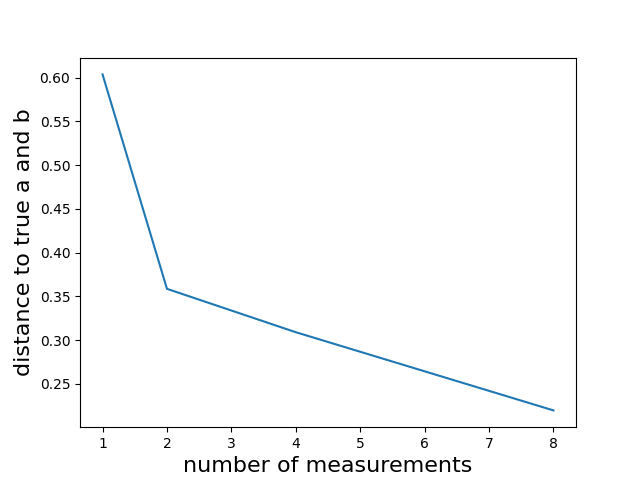}
    \caption{Reverse KL conditional DeepGEM.}
    \label{fig:rev_error}
\end{subfigure}
\caption{Distance to the hyperparameters (a,b) for forward and reverse KL conditional DeepGEM.}
\label{fig:err}
\end{figure}

\begin{figure}
\centering

\begin{subfigure}[t]{.49\textwidth}
\includegraphics[width=\linewidth, scale = 0.9]{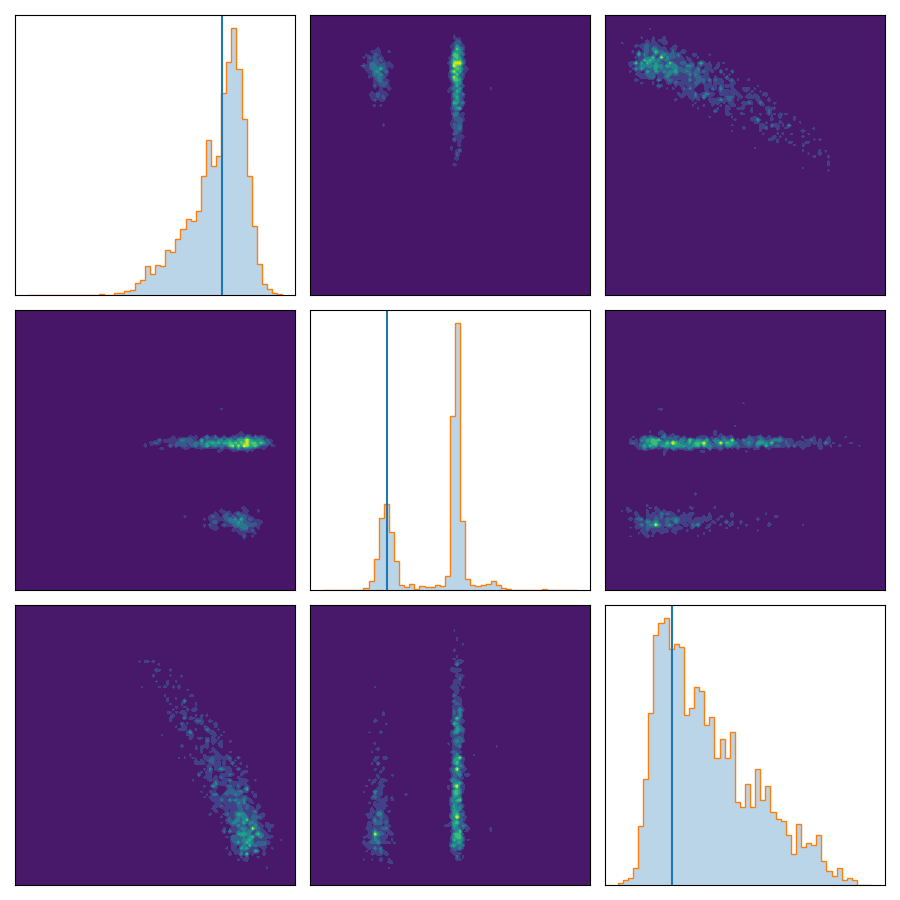}
        \caption{Posterior reconstruction for forward KL conditional DeepGEM for one simulated measurement.}
\end{subfigure}
\begin{subfigure}[t]{.49\textwidth}
\includegraphics[width=\linewidth, scale = 0.9]{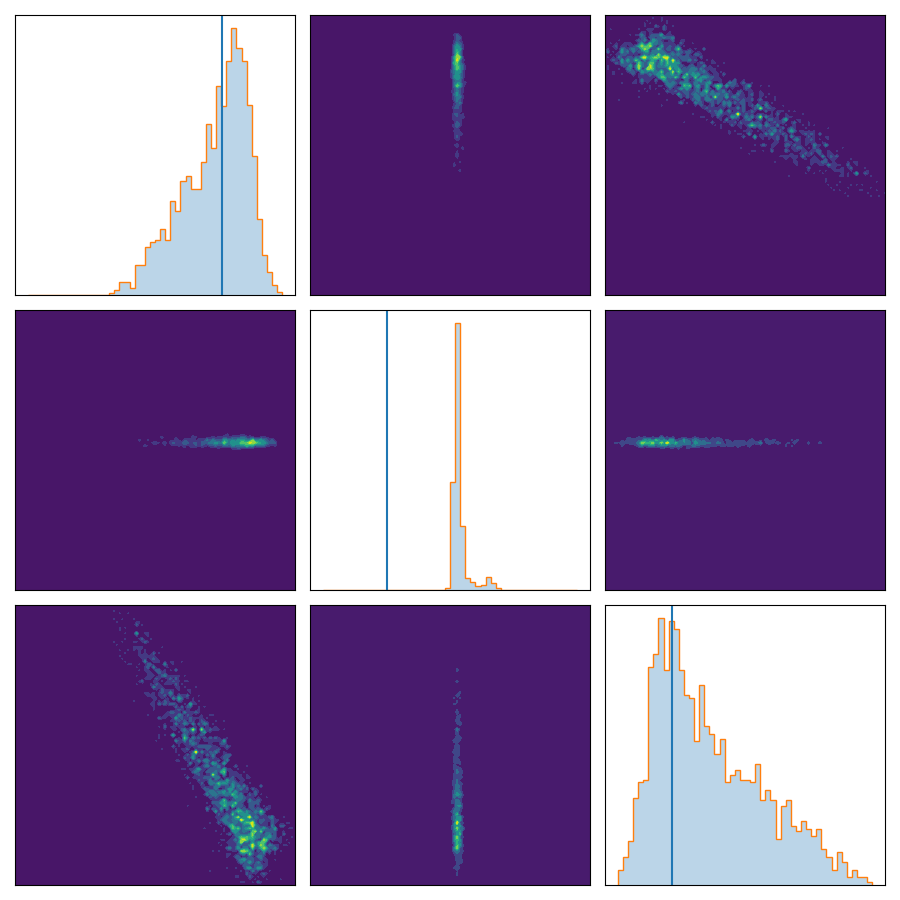}
        \caption{Posterior reconstruction for reverse KL conditional DeepGEM for one simulated measurement.}
\end{subfigure}
\begin{subfigure}[t]{.49\textwidth}
\includegraphics[width=\linewidth, scale = 0.9]{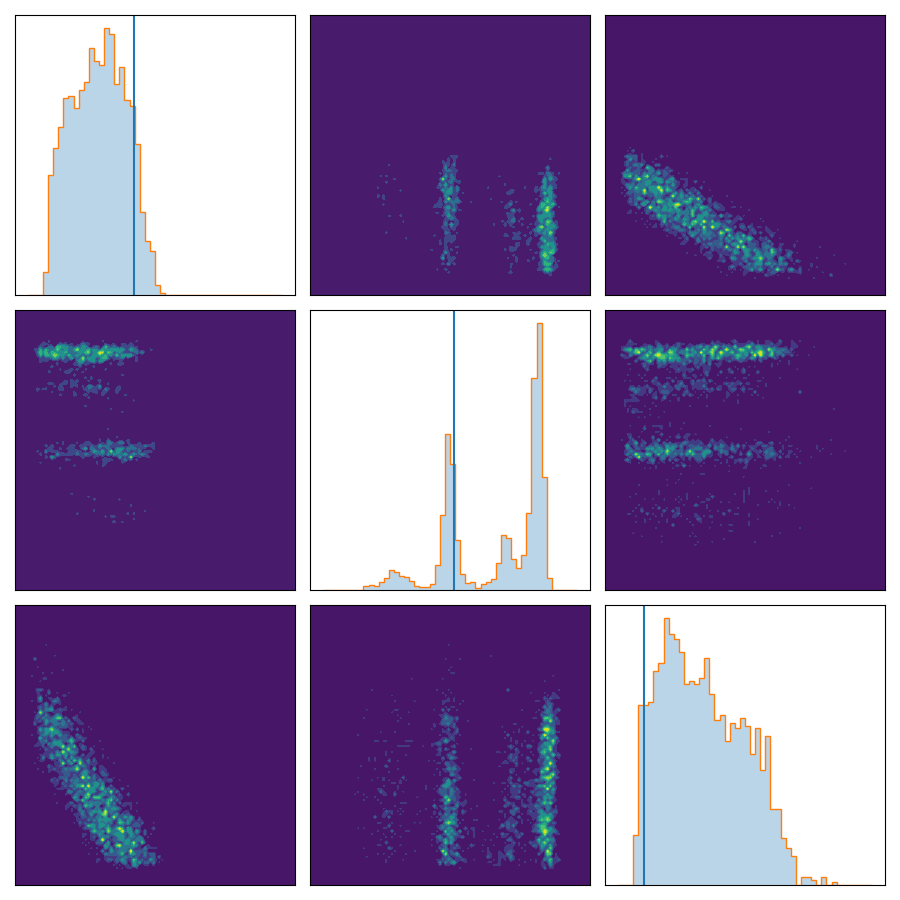}
        \caption{Posterior reconstruction for forward KL conditional DeepGEM for another simulated measurement.}
\end{subfigure}
\begin{subfigure}[t]{.49\textwidth}
\includegraphics[width=\linewidth, scale = 0.9]{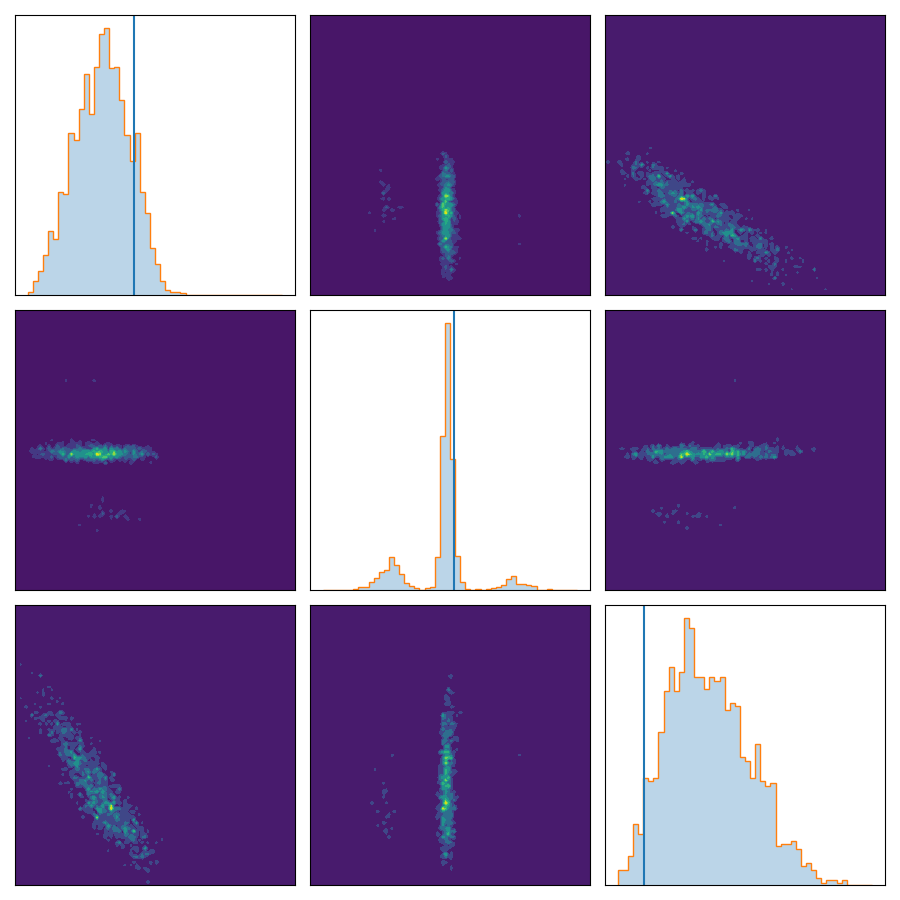}
        \caption{Posterior reconstruction for reverse KL conditional DeepGEM for another simulated measurement.}
\end{subfigure}
\caption{Posterior reconstructions for different measurements using forward/reverse conditional DeepGEM via one dimensional histograms on the diagonal and two dimensional on the offdiagonal. Ground truth x is depicted by the blue line.}
\label{fig_old}
\end{figure}

\subsubsection*{Line grating with oxid layer}
\label{subsec_line_grating}
The second example involves a periodic line grating consisting of a silicon bulk and an oxide layer on  
top. Similar samples were investigated e.g. in \cite{lohr2023nanoscale}. In addition to the geometry parameters, as used in the previous example, the optical constants (OC) of the materials are assumed to be not accurately known. In practice this is often the case if the material composition was changed due to oxidation and contamination of the sample. So  for each material, there are two parameters for the complex refractive index (real and imaginary part)~\cite{henke1993}, which depend on  the material density. Hence we change the OC by varying the densities of the material, i.e., silicon and silicon-oxide.
This results in two parameters changing the OC and five parameter describing the geometry of the line grating.
The refraction pattern are detected for a single  wavelength of the incoming light beam under an angle of incidence of \mbox{$30^{\circ}$},
 from the sample plane and a set of seven azimuth angles between \mbox{$0^{\circ}$} and \mbox{$6^{\circ}$}, the sample is rotated by in the plane. In sum we obtained 77 simulated intensities and hence end up with  x-dimension seven and  y-dimension 77. For simulations the forward model was solved with the software package JCMsuite~\footnote{https://jcmwave.com/}, based on the FEM which solves a boundary value problem following from the Maxwell’s equations \cite{farchminEfficientApproachGlobal2019}. In order to get a strong response for the OC of the oxide layer we used a wavelength of $12.99\,\si{nm}$, right before the absorption edge~\cite{andrleAnisotropyOpticalConstants2021,henke1993}. For this work we standardized the data from the forward simulation \cite{casfor_2024_10580011} on $[0,1]$ and chose a uniform prior for the $x$-data.

Again, we plot two example posterior distributions calculated.
The distribution shapes seen in \ref{fig:new_posts} clearly reflect 
the sensitivity of  the forward operator against the line height (parameter 0), the silicon oxide density (parameter 4) and non-sensitive against the layer roughness (parameter 6). 
The true $a$ and $b$ were set to $0.03$ and $0.25$ respectively. Again as in the first scatterometry example we can see in Table \ref{new_ab} and \ref{new_elbo} that the forward KL performs a bit better in distance to the true a and b as well as ELBO. Similarly, one can observe that the first x-component, the height can be multimodal, where the reverse KL can indeed miss the mode. This can be observed in Fig. \ref{fig:new_posts}. Similarly, the distance to the true a and b decreases  by adding more simulated measurement values, which can be seen in Fig. \ref{fig:err1}.

\begin{table}[]
\centering
\begin{tabular}{ |c||c|c|c|c| } 
 \hline
 number measurements &1 &2 &4 &8 \\ \hline
 forward & \textbf{0.93}  & \textbf{0.19}  & \textbf{0.09} &  \textbf{0.08}\\  \hline
 reverse & 1.07 & 0.20  & 0.10 & 0.11 \\
 \hline
\end{tabular}
 \caption{Distance of estimated $a$ and $b$ to the true ones over 5 runs.}
 \label{new_ab}
\end{table}
\begin{table}
\centering
\begin{tabular}{ |c||c|c|c|c| } 
 \hline
 number measurements &1 &2 &4 &8 \\ \hline
 forward & \textbf{195.9}  & 196.0 & \textbf{190.5} & \textbf{188.1} \\  \hline
 reverse & 195.8  & \textbf{196.1}  & 189.1 & 186.9 \\
 \hline
\end{tabular}
 \caption{ELBO of the algorithms over 5 runs. Calculated for the measurements based on 10000 samples.}
 \label{new_elbo}
\end{table}

\begin{figure}
\begin{subfigure}[t]{.49\textwidth}
\includegraphics[width=\linewidth]{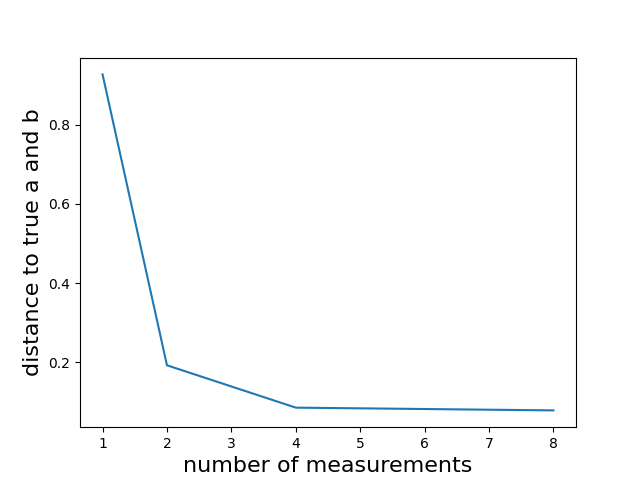}
        \caption{Forward KL conditional DeepGEM.}
        \label{fig:forward_error_new}
        \end{subfigure}
\begin{subfigure}[t]{.49\textwidth}
\includegraphics[width=\linewidth]{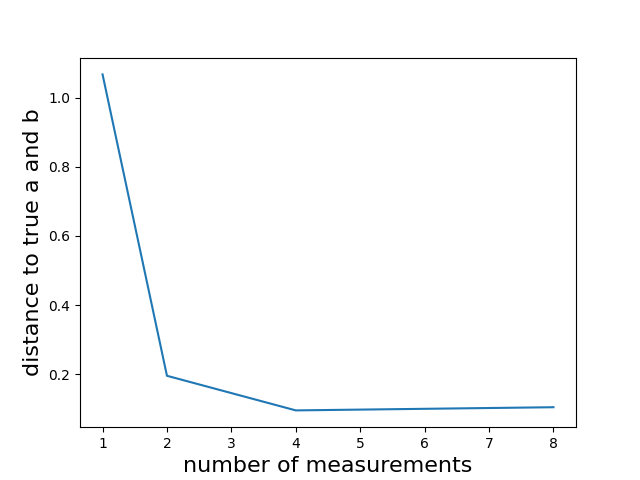}
    \caption{Reverse KL conditional DeepGEM.}
    \label{fig:rev_error_new}
\end{subfigure}
\caption{Distance to the hyperparameters (a,b) for forward and reverse KL conditional DeepGEM.}
\label{fig:err1}
\end{figure}

\begin{figure}
\centering
\begin{subfigure}[t]{.49\textwidth}
\includegraphics[width=\linewidth]{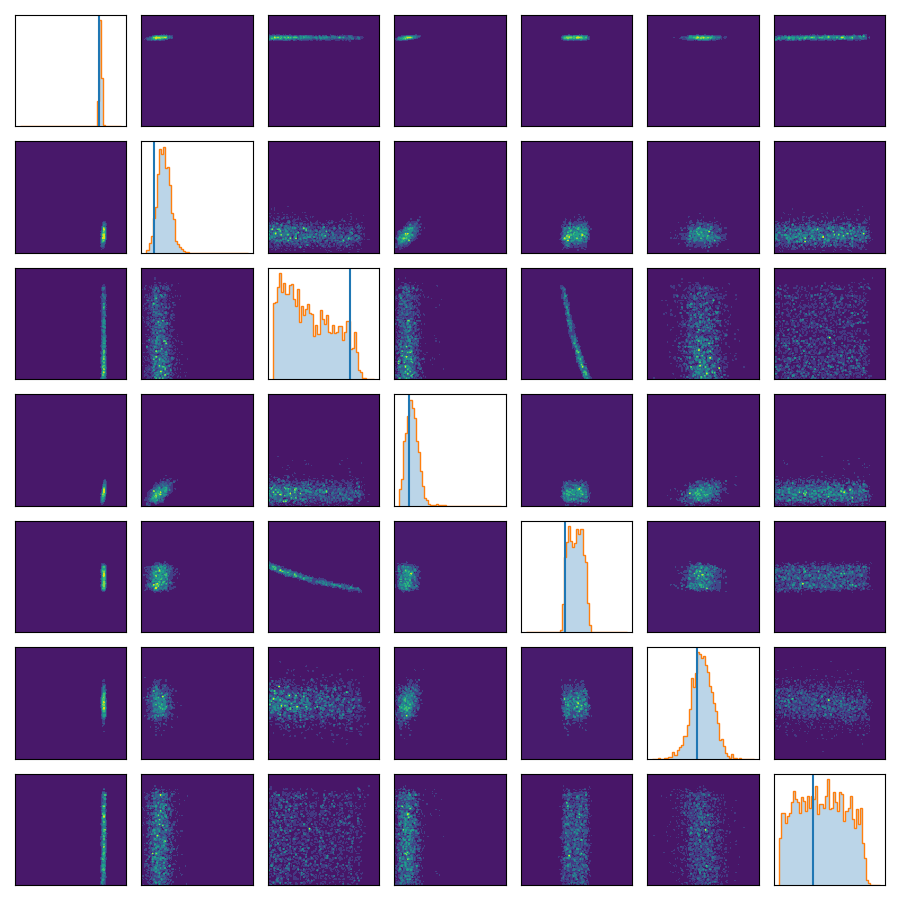}
        \caption{Posterior reconstruction for forward KL conditional DeepGEM for one simulated measurement.}
\end{subfigure}
\begin{subfigure}[t]{.49\textwidth}
\includegraphics[width=\linewidth]{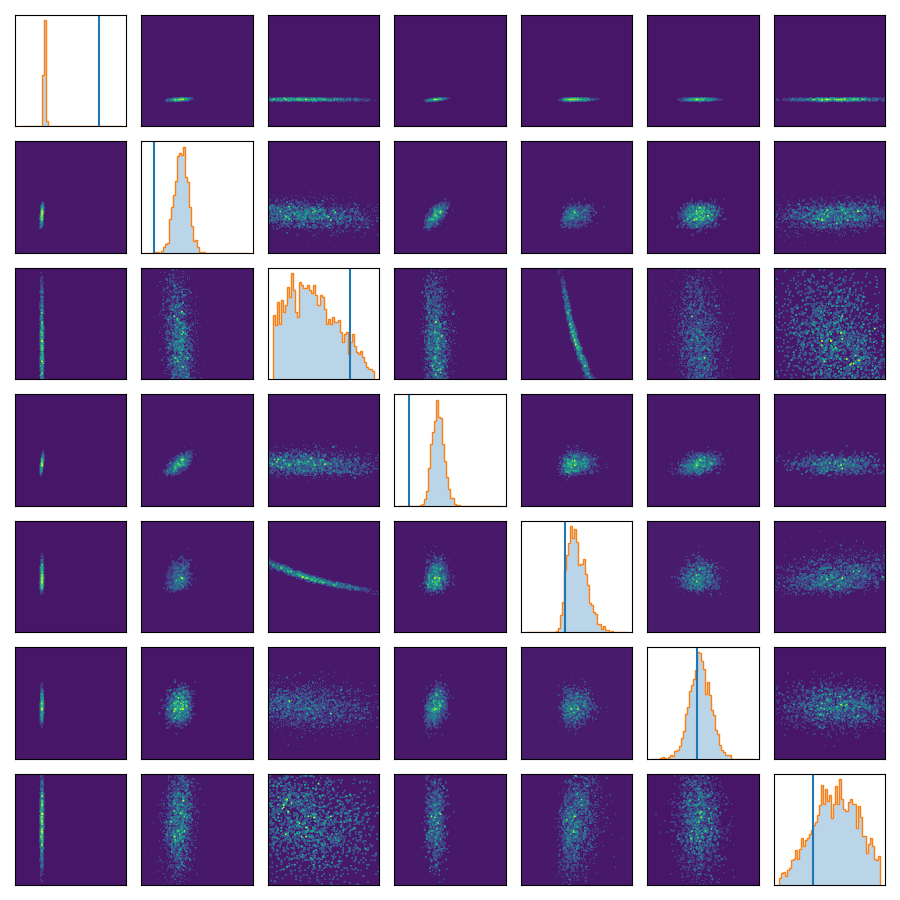}
        \caption{Posterior reconstruction for reverse KL conditional DeepGEM for one simulated measurement.}
\end{subfigure}
\begin{subfigure}[t]{.49\textwidth}
\includegraphics[width=\linewidth]{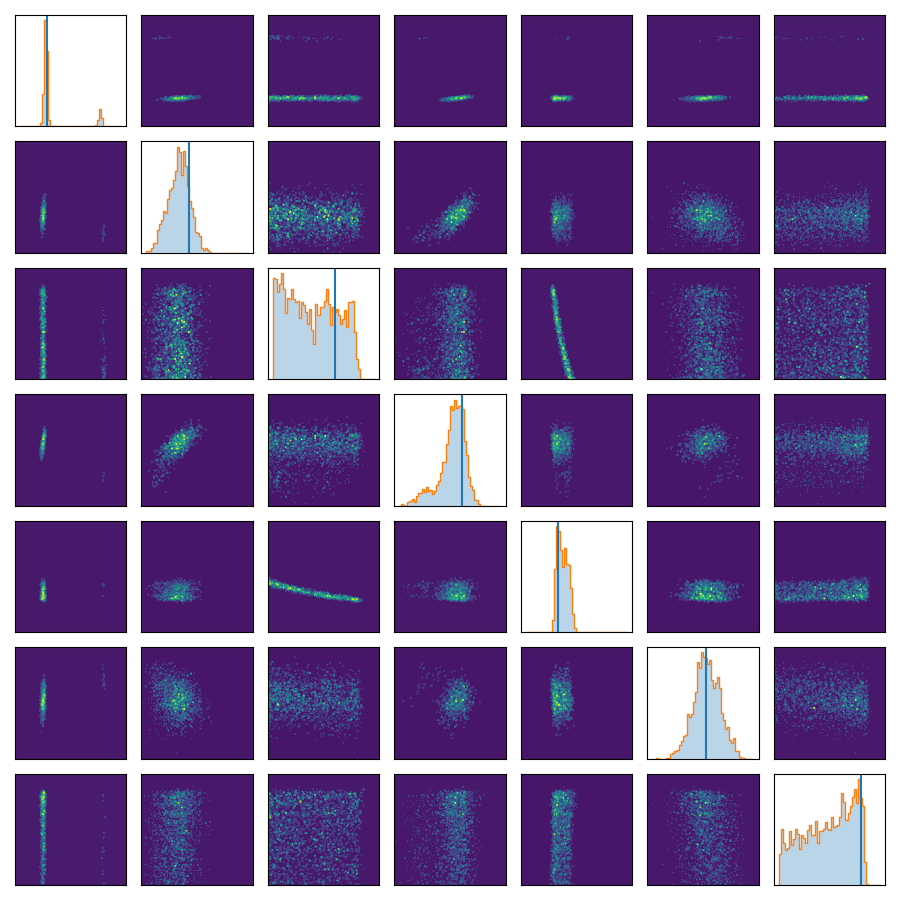}
        \caption{Posterior reconstruction for forward KL conditional DeepGEM for another simulated measurement.}
\end{subfigure}
\begin{subfigure}[t]{.49\textwidth}
\includegraphics[width=\linewidth]{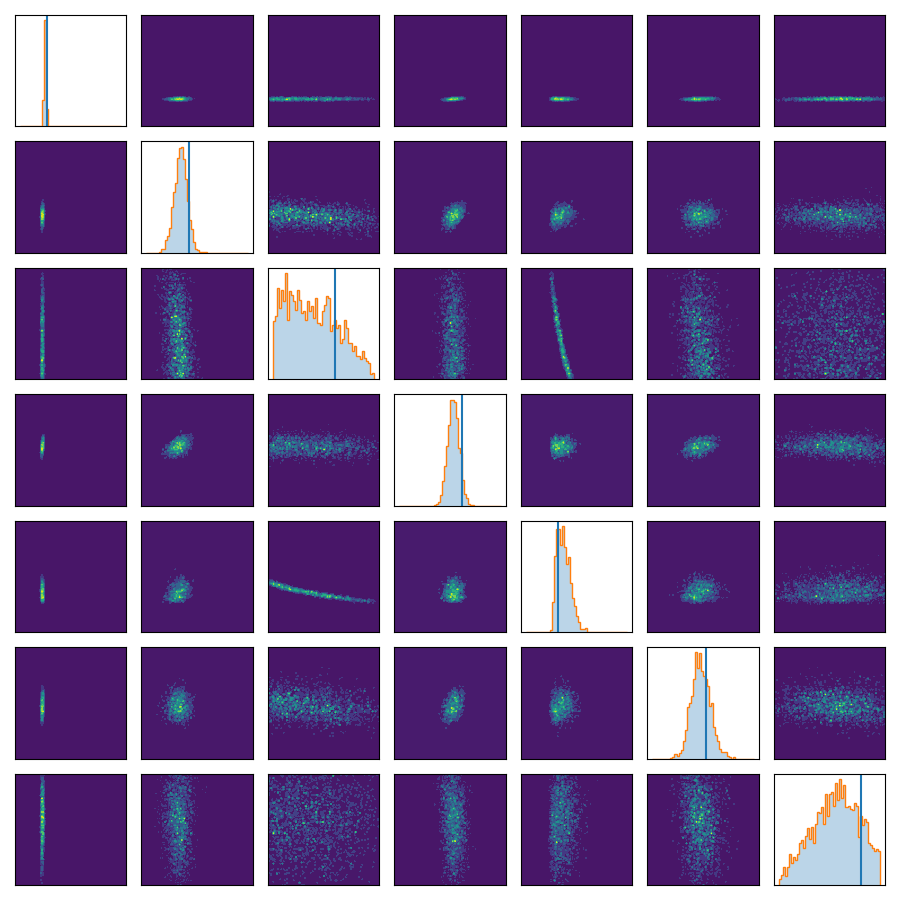}
        \caption{Posterior reconstruction for reverse KL conditional DeepGEM for another simulated measurement.}
\end{subfigure}
\caption{Posterior reconstructions for different measurements using forward/reverse conditional DeepGEM.}
\label{fig:new_posts}
\end{figure}

\section{Conclusions and Limitations} \label{sec:conc}
We developed a nested EM algorithms, one for estimating the posterior distribution via a conditional NF and a second one to solve the M-step within the former EM algorithm 
to estimate the  error model parameters. 
For the special kind of non-additive noise appearing in our applications we 
derived analytic formulas for the inner E- and M-steps. 
We showed advantages of using the forward KL for modelling multimodal distributions. The reverse KL often led to mode collapse. However, there has been a plethora of literature tackling this issue of the reverse KL, namely \cite{AnnealedFlowTransport2021,midgley2023flow, Vaitl_2022, mate2023learning}. It would be interesting to compare these approaches to the forward KL.
Moreover, we could replace the conditional normalizing flow by other methods for posterior sampling like score-based diffusion models \cite{HJA2020,SDBK2023}, conditional GANs \cite{MO2014} or posterior MMD flows \cite{HHABCS2023}.
Furthermore, we chose synthetic $a_{true}$ and $b_{true}$. One of the next steps is to test these approaches on real world measurements. Even if the novel algorithm was applied to two specific real world experiments, it may have an impact to a wide range of applications where indirect measurements are involved. The extension of the algorithm to other noise distributions than Gaussian is analogous. An advantage over standard approaches like Markov-Chain Monte Carlo methods is the fact that once the network has been trained, further similar measurements can be evaluated very quickly. This benefit opens the possibility of scatterentry and similar measurement techniques for real time applications, e.g. important for process control.
In terms of \emph{limitations}, it would be also interesting to test the algorithm on other inverse problems. Intuitively, we believe that the scatterometric inverse problem is particularly well-suited for these estimations since the observed $f(X)$-data is living on a low-dimensional manifold in a nominal high-dimensional space. One can indeed easily think of an inverse problem, where recovering noise parameters is much harder if the observed data already lies in the full space. 

\subsection*{Acknowledgements}
P.H. acknowledges support from the DFG within the SPP 2298 "Theoretical Foundations of Deep Learning" (STE 571/17-1). J.H.~acknowledges funding by the EPSRC programme grant ``The Mathematics of Deep Learning'' with reference EP/V026259/1.
M.C. and S.H. acknowledges the support of the EMPIR project 22IND04-ATMOC. This project (20IND04 ATMOC) has received funding from the EMPIR programme cofinanced by the Participating States and from the European Union’s Horizon 2020 research and innovation programme.

\subsection*{Data availability statement}
The code and data in the form of surrogate networks that support the findings of this study are openly available. This means, the surrogates and the code for running and reproducing the experiments is available under https://github.com/PaulLyonel/ConditionalDeepGEM. For the second photo mask problem we have training data from the forward simulation/simulation of the measurement under https://doi.org/10.5281/zenodo.10580011.

\bibliography{ref}
\appendix
\section{Derivation of the inner M-step}\label{m_step_deriv}
For the simplicity of the notation, we use the abbreviation
\begin{align}\label{eq_def_Qi_nle}
Q_i=P_{V_{\theta^{(r)}}|(X,Y_{\theta^{(r)}})=(x_i,y_i)}=\mathcal N\Big(\tfrac{(a^{(r)})^2(y_i-F(x_i))}{(a^{(r)})^2+(b^{(r)})^2F(x_i)^2},\mathrm{diag}\big(\tfrac{(a^{(r)})^2(b^{(r)})^2 F(x_i)^2}{(a^{(r)})^2+(b^{(r)})^2 F(x_i)^2}\big)\Big).
\end{align}
Using the decomposition \eqref{eq_elbo_decomp} of the ELBO, and noting that the second summand within \eqref{eq_elbo_decomp} does
not depend on the parameters $\theta = (a,b)$, we obtain that the optimization problem \eqref{eq_mstep} is equivalent to
\begin{equation}\label{eq_opt_mstep_nle}
\begin{aligned}
\theta^{(r+1)}&=\argmax_{\theta\in\Theta}\sum_{i=1}^N \frac{1}{N}\E_{v\sim Q_i}[\log(p_{V_\theta,X,Y_\theta}(v,x_i,y_i))]\\
&=\argmax_{\theta\in\Theta}\sum_{i=1}^N \frac{1}{N}\E_{v\sim Q_i}[\log(p_{V_\theta,Y_\theta|X=x_i}(v,y_i))]+\log(p_X(x_i))\\
&=\argmax_{\theta\in\Theta}\sum_{i=1}^N \frac{1}{N}\E_{v\sim Q_i}[\log(p_{V_\theta,Y_\theta|X=x_i}(v,y_i))]
\end{aligned}
\end{equation}
Now, the objective function reads as
\begin{align}
&\quad\sum_{i=1}^N \frac{1}{N}\E_{v\sim Q_i}[\log(p_{V_\theta,Y_\theta|X=x_i}(v,y_i))]\\
&=\sum_{i=1}^N \frac{1}{N}\E_{v\sim Q_i}[\log(p_{Y_\theta|V_\theta=v,X=x_i}(y_i))+\log(p_{V_\theta|X=x_i}(v))].
\end{align}
As it holds by definition that 
$p_{Y_\theta|V_\theta=v,X=x_i}(y_i)=\mathcal N(y_i|F(x_i)+v,b^2\mathrm{diag}(F(x))^2)$
and 
$p_{V_\theta|X=x_i}(v)=p_{V_\theta}(v)=\mathcal N(v|0,a^2 I_n)$, 
this is equal to
\begin{align}
&\quad\sum_{i=1}^N \frac{1}{N}\E_{v\sim Q_i}[\log(\mathcal N(y_i|F(x_i)+v,b^2\mathrm{diag}(F(x_i))^2)+\log(\mathcal N(v|0,a^2 I))]
=A_1(b)+A_2(a),
\end{align}
where (leaving out constants with respect to a or b)
\begin{align*}
A_1(b)&= \frac{1}{N} \sum_{i=1}^N\E_{v\sim Q_i}[\log(\mathcal N(y_i|F(x_i)+v,b^2\mathrm{diag}(F(x_i))^2)]\\
&=-\frac{1}{2Nb^2}\sum_{i=1}^N\sum_{j=1}^n\frac{1}{F_j(x_i)^2}\Big((y_{ij}-F_j(x_i))^2
-2\E_{v\sim Q_i}[v_j](y_{ij}-F_j(x_i))+\E_{v\sim Q_i}[v_j^2]\Big)\\
&\quad-\frac{n}{2}\log(b^2),
\end{align*}
and
\begin{align}
A_2(a)&=\frac{1}{N}\sum_{i=1}^N  \E_{v\sim Q_i}[\log(\mathcal N(v|0,a^2 I))]=-\frac{1}{2a^2\ N}\sum_{i=1}^N\sum_{j=1}^n\E_{v\sim Q_i}[v_j^2]-\frac{n}{2}\log(a^2).
\end{align}
Now, by \eqref{eq_def_Qi_nle}, the expressions $\E_{v\sim Q_i}[v_j]$ and $\E_{v\sim Q_i}[v_j^2]$ are the first and second
moment of certain normal distributions, such that
$$
\E_{v\sim Q_i}[v_j]=\frac{(a^{(r)})^2(y_{ij}-F_j(x_i))}{(a^{(r)})^2+(b^{(r)})^2F_j(x_i)^2}
$$
and
$$
\E_{v\sim Q_i}[v_j^2]=\Big(\frac{(a^{(r)})^2(y_{ij}-F_j(x_i))}{(a^{(r)})^2+(b^{(r)})^2F_j(x_i)^2}\Big)^2
+\frac{(a^{(r)})^2(b^{(r)})^2 F_j(x_i)^2}{(a^{(r)})^2+(b^{(r)})^2F_j(x_i)^2}.
$$
Putting everything together, we obtain that \eqref{eq_opt_mstep_nle} is equivalent to
$$
\theta^{(r+1)}=\argmax_{(a,b) \in \R^2_{\ge 0}} A_1(b)+A_2(a)
$$
with
$$
 A_1(b)=\frac{1}{2b^2}c_1^{(r)}-n \log(b)+\mathrm{const},\quad A_2(a)=\frac{1}{2a^2}c_2^{(r)}-n \log(a)+\mathrm{const},
$$
where $\mathrm{const}$ denotes an unspecified constant independent of $a$ and $b$.
Further, the $c_i^{(r)}$ are given by
\begin{align}
c_2^{(r)}=-\frac{1}{N}\sum_{i=1}^N \bigg(\sum_{j=1}^n \Big(\frac{(a^{(r)})^2(y_{ij}-F_j(x_i))}{(a^{(r)})^2+(b^{(r)})^2F_j(x_i)^2}\Big)^2+\frac{(a^{(r)}b^{(r)}F_j(x_i))^2}{(a^{(r)})^2+(b^{(r)}F_j(x_i))^2}\bigg),
\end{align}
and
\begin{align}
c_1^{(r)}&=-\frac{1}{N}\sum_{i=1}^N\sum_{j=1}^n\frac{1}{F_j(x_i)^2}
\Big((y_{ij}-F_j(x_i))^2-2\frac{(a^{(r)})^2(y_{ij}-F_j(x_i))^2}{(a^{(r)})^2+(b^{(r)})^2F_j(x_i)^2}\\
&\quad +\Big(\frac{(a^{(r)})^2(y_{ij}-F_j(x_i))}{(a^{(r)})^2+(b^{(r)})^2F_j(x_i)^2}\Big)^2
+\frac{(a^{(r)}b^{(r)}F_j(x_i))^2}{(a^{(r)})^2+(b^{(r)}F_j(x_i))^2}\Big).
\end{align}
Bringing  the first three terms onto one denominator, this can be simplified to
\begin{align}
c_1^{(r)}
&=-\frac{1}{N}\sum_{i=1}^N\sum_{j=1}^n\frac{(y_{ij}-F_j(x_i))^2(b^{(r)})^4F_j(x_i)^2}{((a^{(r)})^2
+(b^{(r)}F_j(x_i))^2)^2}+\frac{(a^{(r)}b^{(r)})^2}{(a^{(r)})^2+(b^{(r)}F_j(x_i))^2}.
\end{align}
Note that by definition $c_i^{(r)}$, $i=1,2$ are non-positive.
Thus, $a^{(r+1)}$ and $b^{(r+1)}$ are given by
$$
a^{(r+1)}=\argmax_{a\geq 0} A_2(a),
\qquad 
b^{(r+1)}=\argmax_{b\geq 0} A_1(b).
$$
By setting the derivatives of $A_1$ and $A_2$ to zero, this is equivalent to
$$
a^{(r+1)}=-\frac{c_2^{(r)}}{n},\qquad b^{(r+1)}=-\frac{c_1^{(r)}}{n}.
$$

\section{Convergence plots of a and b}
Here we showcase the convergence of the parameters $a$ and $b$ in the first scatterometry photo mask example. For one particular run the convergence of $a$ and $b$ is shown. Note that both methods seem to converge to the true $a$ and $b$ values, however the reverse KL trains more unstably, which might be tackable with more careful hyperparameter selection or other stabilization techniques. 

\begin{figure}
\begin{subfigure}[t]{.49\textwidth}
\includegraphics[width=\linewidth]{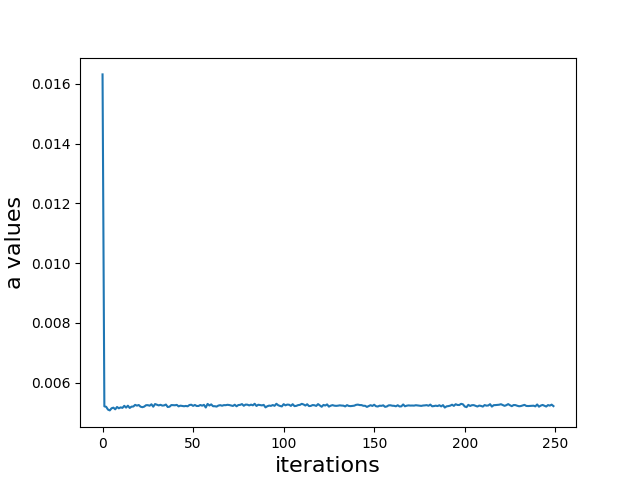}
        \caption{Forward KL cond. DeepGEM plot for a.}
        \end{subfigure}
        \begin{subfigure}[t]{.49\textwidth}
\includegraphics[width=\linewidth]{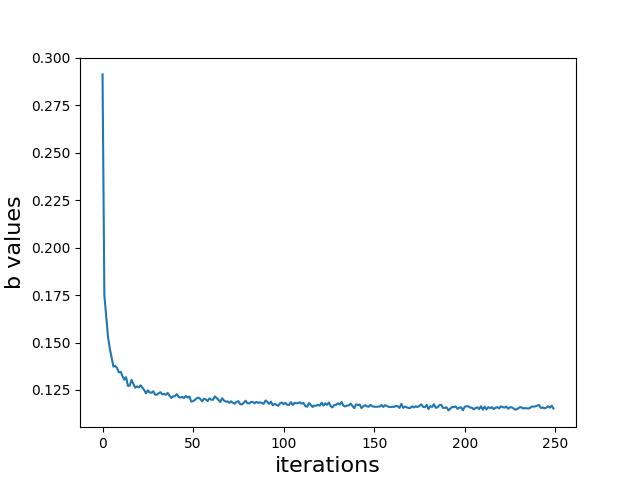}
        \caption{Forward KL cond. DeepGEM plot for b.}
        \end{subfigure}
\begin{subfigure}[t]{.49\textwidth}
\includegraphics[width=\linewidth]{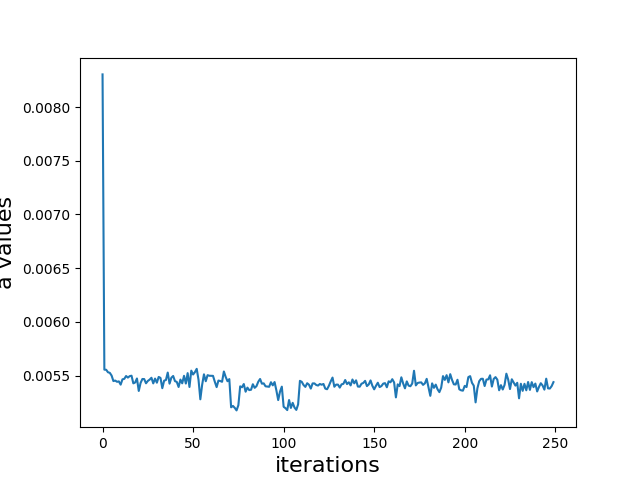}
        \caption{Reverse KL cond. DeepGEM plot for a.}
        \end{subfigure}
        \begin{subfigure}[t]{.49\textwidth}
\includegraphics[width=\linewidth]{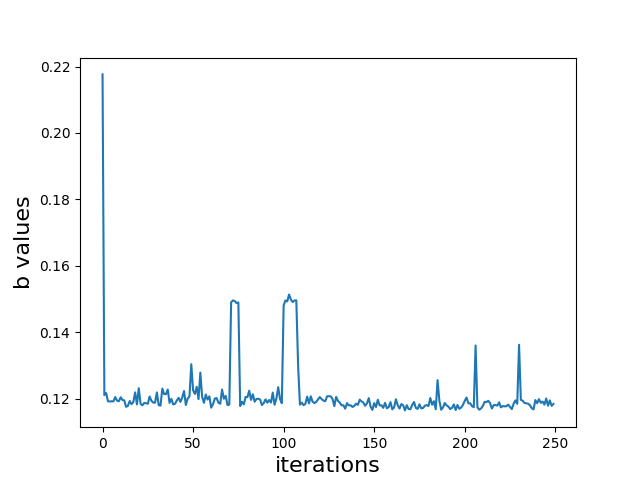}
        \caption{Reverse KL cond. DeepGEM plot for b.}
        \end{subfigure}
\caption{Convergence plots for $a$ and $b$ where we save every 20 EM-steps.}
\label{fig:abconv}
\end{figure}

\section{Sensitivity analysis of the forward model}
\label{sec:sens_ana}

\begin{figure}
    \centering
    \includegraphics[width= 0.6\linewidth]{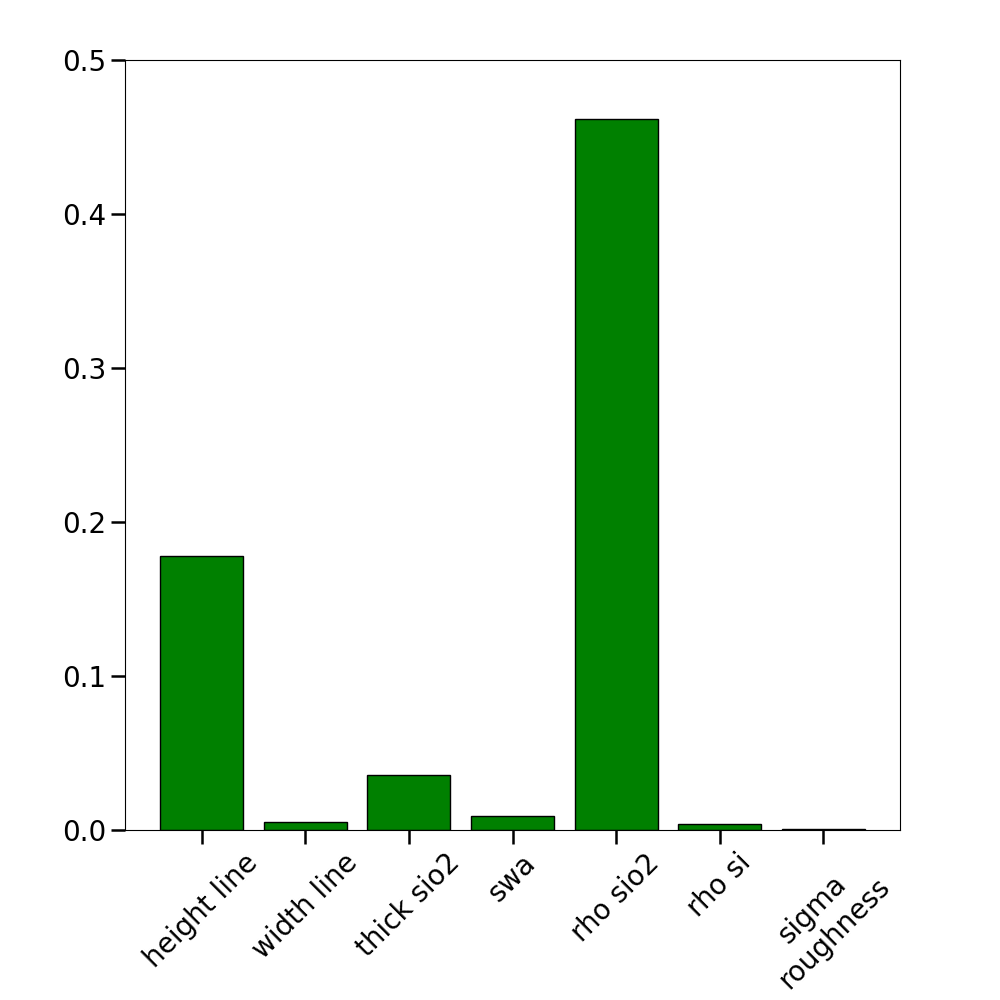}
    \caption{Barplot of Sobol` indices for PC-expansion of the forward model.}
    \label{fig:sobol_indices_line_grating}
\end{figure}

To verify our results for the reconstruction of the line grating in \ref{subsec_line_grating}, we make a sensitivity analysis of the corresponding forward model. For the sensitivity analysis we have to determine the Sobol` indices, which are coefficients of a decomposition of the forward model and describe the impact of each parameter combination on the forward model. Those indices are normalized to $[0,1]$ and sum up to $1$ \cite{sobolGlobalSensitivityIndices2001a,SensitivityEstimatesNonlinear1993}. Making an approximation of the forward model in a polynomial basis using Polynomial Chaos (PC) \cite{wienerHomogeneousChaos1938a,farchminEfficientApproachGlobal2019} makes it very easy to calculate the Sobol` indices.
The indices in Fig.\,\ref{fig:sobol_indices_line_grating} come from a PC-approximation with a relative $L_2$-error of about $0.076$ and show the dependence on each single parameter. It is clearly seen, that the height of the grating line and the density of the oxide layer and hence the OC for silicon-oxide has a huge impact on the forward model. In general the sensitivity analysis fits very well to the reconstruction of the line grating, since in Fig.\,\ref{fig:new_posts} the distributions for parameter $0$ and $4$ are very sharp defined, while those which are very broad distributed also show a low impact on the forward model. For the sensitivity analysis the source software tool PyThia was used \cite{hegemann2023pythia}. 

\end{document}